\documentclass{article}

\usepackage{microtype}
\usepackage{graphicx}
\usepackage{subfigure}
\usepackage{booktabs} 
\usepackage{amsmath}
\usepackage{amsthm}
\usepackage{bm}
\usepackage{amsfonts}
\usepackage{amssymb}
\usepackage[resetlabels,labeled]{multibib}
\newcites{S}{Supplement References}
\newtheorem{proposition}{Proposition}

\usepackage{hyperref}

\usepackage[accepted]{icml2020}

\icmltitlerunning{Deep GMRFs}

\begin{document}

\twocolumn[
\icmltitle{Deep Gaussian Markov Random Fields}




\begin{icmlauthorlist}
\icmlauthor{Per Sid\'en}{liu}
\icmlauthor{Fredrik Lindsten}{liu}
\end{icmlauthorlist}

\icmlaffiliation{liu}{Division of Statistics and Machine Learning, Department of Computer and Information Science, Link\"{o}ping University, Link\"{o}ping, Sweden}

\icmlcorrespondingauthor{Per Sid\'en}{per.siden@liu.se}


\vskip 0.3in
]



\printAffiliationsAndNotice{}  
\begin{abstract}
Gaussian Markov random fields (GMRFs) are probabilistic graphical models widely used in spatial statistics and related fields to model dependencies over spatial structures. We establish a formal connection between GMRFs and convolutional neural networks (CNNs). Common GMRFs are special cases of a generative model where the inverse mapping from data to latent variables is given by a 1-layer linear CNN. This connection allows us to generalize GMRFs to multi-layer CNN architectures, effectively increasing the order of the corresponding GMRF in a way which has favorable computational scaling. We describe how well-established tools, such as autodiff and variational inference, can be used for simple and efficient inference and learning of the deep GMRF. We demonstrate the flexibility of the proposed model and show that it outperforms the state-of-the-art on a dataset of satellite temperatures, in terms of prediction and predictive uncertainty.
\end{abstract}

\section{Introduction}
\label{introduction}
Convolutional neural networks (CNNs) are the \emph{de facto} standard model in computer vision when training on a large set of images. 
Images are lattice-based data with local dependencies and thus have clear connections with spatial statistics. However, many spatial problems lack the abundance of data common to computer vision applications, and often we need to build a model based on a single ``image'', i.e., data recorded over some spatial field. Models such as deep image prior \cite{ulyanov2018deep} have shown that CNN architectures can encode useful spatial priors even in such situations, but the dominant approach is still to model the spatial dependencies explicitly using, e.g., Gaussian processes (GPs) \cite{williams2006gaussian} or Gaussian Markov random fields (GMRFs) \cite{Isham2004}.

In this paper we show a formal connection between GMRFs applied to lattice data and CNNs. Common GMRFs based on nearest neighbour interactions can be viewed as special cases of a generative model where the inverse mapping from the spatial field $\mathbf{x}$ to a latent variable $\mathbf{z}$ is given by a 1-layer linear CNN. Since common GP models have previously been shown to be tightly linked with GMRFs \cite{Lindgren2011}, this connection applies to certain GPs as well.

Modeling the inverse mapping ($\mathbf{x} \rightarrow \mathbf{z}$) using a CNN results in an auto-regressive (AR) spatial model, whereas using a CNN for the forward mapping ($\mathbf{z} \rightarrow \mathbf{x}$) would correspond to a moving average (MA) model (see, e.g., \cite{Ljung:1999} for a discussion on AR and MA models in a time series context).
This has the important implication that we obtain an infinite receptive field (i.e., infinite range on the spatial dependencies in $\mathbf{x}$) even with a 1-layer model. Indeed, this is a well known property of GMRFs.

The interpretation of a GMRF as a 1-layer CNN opens up for a straightforward generalization to multi-layer architectures, resulting in \emph{deep GMRFs} (DGMRFs). Even when all layers are linear this has important practical implications: adding layers corresponds to increasing the auto-regressive order of the GMRF, i.e., edges are added to the GMRF graph which improves its expressivity. For conventional GMRF algorithms, simply adding more edges can have have a significant impact on the computational complexity due to a reduced degree of sparsity of the resulting precision matrix. For a DGMRF, on the other hand, the structure imposed by a multi-layer CNN architecture results in a favorable computational scaling. Indeed, using variational inference for the latent variables, we obtain a learning algorithm that scales linearly both with the number of dimensions in the spatial field (``pixels'') and the number of layers (AR order). Furthermore, viewing GMRFs through the lens of deep learning allows us to use well-established toolboxes for, e.g., automatic differentiation and GPU training, to facilitate simple learning of DGMRFs.

After a review of GMRFs in Section~\ref{sec:Background}, we introduce the DGMRF model in Section~\ref{sec:deepGMRFs}. Section~\ref{sec:Inference} describes how to efficiently train the model, and how to compute the posterior predictive distribution, including uncertainty. We discuss related work in Section~\ref{sec:Related-Work}. The results in Section~\ref{sec:Results} illustrate how DGMRFs are adaptive models, with outstanding predictive ability. Section~\ref{sec:Conclusions} concludes.


\section{Background}\label{sec:Background}

\subsection{Gaussian Markov Random Fields}
A Gaussian Markov random field (GMRF) $\mathbf{x}$ is an $N$-dimensional
Gaussian vector with mean $\bm{\mu}$ and precision (inverse covariance) matrix $\bm{Q}$, so that 
\(
\mathbf{x}\sim\mathcal{N}\left(\boldsymbol{\mu},\mathbf{Q}^{-1}\right).
\)
For each GMRF there exists a graph $\mathcal{G}=\left(\mathcal{V},\mathcal{E}\right)$,
with vertices $\mathcal{V}$ that correspond to the elements in $\mathbf{x}$,
and edges $\mathcal{E}$ that define their conditional independencies.
For example, the vertices could represent the
pixels in an image, with the edges connecting neighboring pixels.
Formally,
\[
\left\{ i,j\right\} \notin\mathcal{E}\Longleftrightarrow x_{i}\perp x_{j}|\mathbf{x}_{-ij},\,\,\,\,\,\text{for all\,\,\,}i\neq j,
\]
where $\mathbf{x}_{-ij}$ refers to all elements except $i$ and $j$,
meaning that two elements $x_{i}$ and $x_{j}$, that are not neighbors
in the graph, are conditionally independent given the rest. 
Importantly, the edges $\mathcal{E}$ also determine the zero-pattern
in the precision matrix $\mathbf{Q}$, as every GMRF has the property
\[
\left\{ i,j\right\} \in\mathcal{E}\Longleftrightarrow Q_{ij}\neq0,\,\,\,\,\,\text{for all\,\,\,}i\neq j.
\]
This means that a sparsely connected graph $\mathcal{G}$ results
in a sparse precision matrix $\mathbf{Q}$, which gives great computational
gains compared to working with the dense covariance matrix in many
large-scale applications.

\subsection{Example: GMRF Defined Using Convolution}
As an example, consider the second-order intrinsic GMRF or thin-plate
spline model \citep{Isham2004}, which can be defined by
\(
\mathbf{x}\sim\mathcal{N}\left(\mathbf{0},\left(\mathbf{G}^\top\mathbf{G}\right)^{-1}\right),
\)
with
\begin{equation}
G_{ij}=\begin{cases}
4 & ,\,\text{for}\,i=j\\
-1 & ,\,\text{for}\,i\sim j\\
0 & ,\,\text{otherwise}
\end{cases},\label{eq:G_ij}
\end{equation}
where $i\sim j$ denotes 
adjacency\footnote{It is perhaps more standard to define $G_{ii}$ equal to the number
of neighbors of pixel $i$, which makes a difference in the image
border. Our definition is convenient here as it makes $\mathbf{G}$ invertible.}. Imputing missing pixel values conditioned on its second-order neighborhood
using this prior is equivalent to bicubic interpolation. The non-zero
elements of each row $i$ of $\mathbf{G}$ and $\mathbf{Q}$ , with
respect to neighboring pixels in 2D, can be compactly represented
through the stencils
\begin{align}
&\mathbf{w}_{\mathbf{G}}: \;
\left[
\begin{smallmatrix}
 & -1\\
-1 & 4 & -1\\
 & -1
 \end{smallmatrix}
\right]
&&\mathbf{w}_{\mathbf{Q}}: \;
\left[
\begin{smallmatrix}
 &  & 1\\
 & 2 & -8 & 2\\
1 & -8 & 20 & -8 & 1\\
 & 2 & -8 & 2\\
 &  & 1
\end{smallmatrix}
\right].
\label{eq:stencils}
\end{align}
An equivalent definition of this model is to first define $\mathbf{z}\sim\mathcal{N}(\mathbf{0},\mathbf{I})$
and then $\mathbf{x}$ through the inverse transform 
\begin{equation}
\mathbf{z}=\mathbf{G}\mathbf{x}.\label{eq:z=00003DGx}
\end{equation}
It is trivial that $\mathbf{x}$ is Gaussian with mean $\mathbf{0}$
and it can be readily verified that
\(
\text{Cov}(\mathbf{x})=\mathbf{G}^{-1}\mathbf{I}\mathbf{G}^{-\top}=\left(\mathbf{G}^\top\mathbf{G}\right)^{-1}.
\)
In a third, equivalent definition the inverse transform $\mathbf{z}=\mathbf{G}\mathbf{x}$
is instead written as a convolution
\[
\mathbf{Z}=\text{conv}\left(\mathbf{X},\mathbf{w}_{\mathbf{G}}\right),
\]
where $\mathbf{Z}$ and $\mathbf{X}$ are image representations of
the vectors $\mathbf{z}$ and $\mathbf{x}$. The stencil $\mathbf{w}_{\mathbf{G}}$ in Eq.~(\ref{eq:stencils}) is here used as a filter and $\text{conv}()$ denotes \textit{same} convolution (\texttt{padding="SAME"}), for the equivalence to hold. This observation, that a certain kind
of GMRF can be defined with the convolution operator, a filter and
an auxiliary variable is a key observation for DGMRFs, which are defined in Section~\ref{sec:deepGMRFs}.
\subsection{Link Between GMRFs and Gaussian Processes}\label{sec:SPDE}
Another instructive view of GMRFs is as a represetation of a Gaussian
processes (GP) with Matérn kernel \cite{Lindgren2011}. This result
comes from a stochastic partial differential equation (SPDE) of the
form
\[
\left(\kappa^{2}-\Delta\right)^{\gamma}\tau x(\mathbf{s})=W(\mathbf{s}),
\]
which can be shown to have a GP solution $x(\mathbf{s})$
with Matérn covariance function \cite{whittle1954,whittle1963}. Here,
$W(\mathbf{s})$ is Gaussian white noise in a continuous
coordinate $\mathbf{s}$, $\Delta$ is the Laplacian operator, and
$\kappa$, $\tau$ and $\gamma$ are hyperparameters that appear in
the Matérn kernel. In particular, $\gamma$ controls the smoothness
of the GP. Moreover, for positive integer values of $\gamma$, a numerical
finite difference approximation to the SPDE, on the integer lattice,
is given by
\begin{equation}
\tau\left(\kappa^{2}\mathbf{I}+\mathbf{G}\right)^{\gamma}\mathbf{x}=\mathbf{z},\label{eq:discreteSPDE}
\end{equation}
with $\mathbf{G}$ defined as in Eq.~(\ref{eq:G_ij}). As in Eq.~(\ref{eq:z=00003DGx}),
this inverse transform describes a GMRF, here with precision matrix
$\mathbf{Q}=\tau^{2}((\kappa^{2}\mathbf{I}+\mathbf{G})^{\gamma})^\top(\kappa^{2}\mathbf{I}+\mathbf{G})^{\gamma}$.
Firstly, this provides a sparse representation of
a GP as a GMRF, with a discrepancy that can be reduced by making the
discretization of the SPDE finer. Secondly, it gives an interpretation of GMRFs as models with similar properties as GPs.

\subsection{GMRFs as Spatial Priors}
GMRFs are commonly used as spatial priors for latent variables, which is common in spatial statistics and image analysis. In the simplest case, the data $\mathbf{y}$ is modeled as Gaussian, and conditionally independent given $\mathbf{x}$
\begin{align*}
p(\mathbf{y}|\mathbf{x}) & =\prod_{i\in\mathcal{M}}p(y_{i}|x_{i}),
&
y_{i}|x_{i} & \sim\mathcal{N}\left(y_{i}|x_{i},\sigma^{2}\right),
\end{align*}
where $\mathcal{M}\subseteq \{1,\ldots,N\}$ is the set of observed pixels with $M=|\mathcal{M}|$. Typical problems include inpainting ($M<N$), and denoising ($\sigma^2>0$), where the target is to reconstruct the latent $\mathbf{x}$. Conveniently, in this situation the GMRF prior is conjugate, so the posterior is also a GMRF
\begin{align}
\begin{split}
\mathbf{x}|\mathbf{y}&\sim\mathcal{N}(\tilde{\boldsymbol{\mu}},\tilde{\mathbf{Q}}^{-1}),\,\,\,\,\text{with} \\
\tilde{\mathbf{Q}}&=\mathbf{Q}+\frac{1}{\sigma^{2}}\mathbf{I}_\mathbf{m},\,\,\,\,\,\,\,\,\,\tilde{\boldsymbol{\mu}}=\tilde{\mathbf{Q}}^{-1}\left(\mathbf{Q}\boldsymbol{\mu}+\frac{1}{\sigma^{2}}\mathbf{y}\right).\label{eq:posteriorGMRF}
\end{split}
\end{align}
Here, the observation mask $\mathbf{m}$ has value 0 for missing pixels and 1 elsewhere, $\mathbf{y}$ are the observations with value 0 at missing pixels and $\mathbf{I}_\mathbf{m}$ is the identity matrix with diagonal element 0 at missing pixels. Although the posterior is on closed form, the computational cost associated with $\tilde{\mathbf{Q}}^{-1}$, needed for the posterior mean and marginal standard deviations, can be high. This is discussed more in Section~\ref{sec:Inference}.

\section{Deep Gaussian Markov Random Fields\label{sec:deepGMRFs}}
\subsection{Linear DGMRFs}
We define a linear DGMRF $\mathbf{x}$ 
using an auxiliary standard
Gaussian vector $\mathbf{z}$ and a bijective function $\mathbf{g_{\boldsymbol{\theta}}}:\mathbb{R}^{N}\rightarrow\mathbb{R}^{N}$,
\begin{align*}
\mathbf{z} & \sim\mathcal{N}\left(\mathbf{0},\mathbf{I}\right), &
\mathbf{z} & =\mathbf{g}_{\boldsymbol{\theta}}(\mathbf{x}).
\end{align*}
In other words, we define $\mathbf{x}$ through the inverse transform $\mathbf{g}_{\boldsymbol{\theta}}^{-1}$. The function $\mathbf{g_{\boldsymbol{\theta}}}(\mathbf{x})$
is for now assumed to be linear, so we can write $\mathbf{g_{\boldsymbol{\theta}}}(\mathbf{x})=\mathbf{G}_{\boldsymbol{\theta}}\mathbf{x}+\mathbf{b}_{\boldsymbol{\theta}}$, where $\mathbf{G}_{\boldsymbol{\theta}}$ is an invertible square matrix. The non-linear case is considered in Section~\ref{sec:non-linear-deep-GMRFs}. We will specify  $\mathbf{g_{\boldsymbol{\theta}}}(\mathbf{x})$
using a CNN with $L$ layers.
Let $\mathbf{Z}_{l}$ be a tensor of dimension $H\times W\times C$,
with height $H$, width $W$ and $C$ channels, and let $\mathbf{z}_{l}=\text{vec}\left(\mathbf{Z}_{l}\right)$
be its vectorized version, with length $N=HWC$. The output of layer
$l$ is defined as
\begin{equation}
\mathbf{Z}_{l}=\text{conv}\left(\mathbf{Z}_{l-1},\mathbf{w}_{l}\right)+\mathbf{b}_{l},
\label{eq:layer-transition}
\end{equation}
where $\mathbf{w}_{l}$ is a 4D-tensor containing $C\times C$ 2D-filters and $\mathbf{b}_{l}$
is a set of $C$ biases. Here, $\text{conv}()$ refers to multichannel same convolution, 
more details are given in Section~\ref{subsec:multichannel}. In particular, we define $\mathbf{Z}_{L}\triangleq\mathbf{Z}$
and $\mathbf{Z}_{0}\triangleq\mathbf{X}$. An illustration of the model is given in Figure~\ref{fig:plot_model}. The model parameters $\boldsymbol{\theta}=\left\{\left(\mathbf{w}_{l},\mathbf{b}_{l}\right):l=1,\ldots,L\right\} $ will be omitted in the following for brevity.

Just as for normalizing flows \cite{dinh2014nice,rezende15}, 
$\mathbf{g}$ can be seen as a sequence of bijections $\mathbf{g}=\mathbf{g}_{L}\circ\mathbf{g}_{L-1}\circ\cdots\circ\mathbf{g}_{1}$, each with corresponding transform matrix $\mathbf{G}_l$.  Since $\mathbf{g}$ is linear, $\mathbf{x}$ is a GMRF with density
\begin{equation}
p(\mathbf{x})=\frac{\left|\det\left(\mathbf{G}\right)\right|}{\left(2\pi\right)^{N/2}}\exp\left(-\frac{1}{2}\left(\mathbf{x}-\boldsymbol{\mu}\right)^{\top}\mathbf{G}^{\top}\mathbf{G}\left(\mathbf{x}-\boldsymbol{\mu}\right)\right),\label{eq:x-linear-density}
\end{equation}
with $\mathbf{G}=\mathbf{G}_{L}\mathbf{G}_{L-1}\cdots\mathbf{G}_{1}$ and the mean $\boldsymbol{\mu}=-\mathbf{G}^{-1}\mathbf{b}$ where $\mathbf{b}$ can be computed as $\mathbf{b}=\mathbf{g}(\mathbf{0})$.  The determinant $\det(\mathbf{G})$ can be computed as
\begin{equation}
\textstyle
\det\left(\mathbf{G}\right)=\prod_{l=1}^{L}\det\left(\mathbf{G}_{l}\right)
\label{eq:detG}
\end{equation}
and we address the problem of making this computation fast below in Section~\ref{sec:conv-filters}.
This GMRF has precision matrix $\mathbf{Q}=\mathbf{G}^\top\mathbf{G}$, that is guaranteed to be positive (semi-)definite for all $\boldsymbol{\theta}$, which gives an advantage compared to modeling $\mathbf{Q}$ directly.
\begin{figure}
\begin{center}
\includegraphics[width=86mm,trim={30mm 30mm 80mm 20mm},clip]{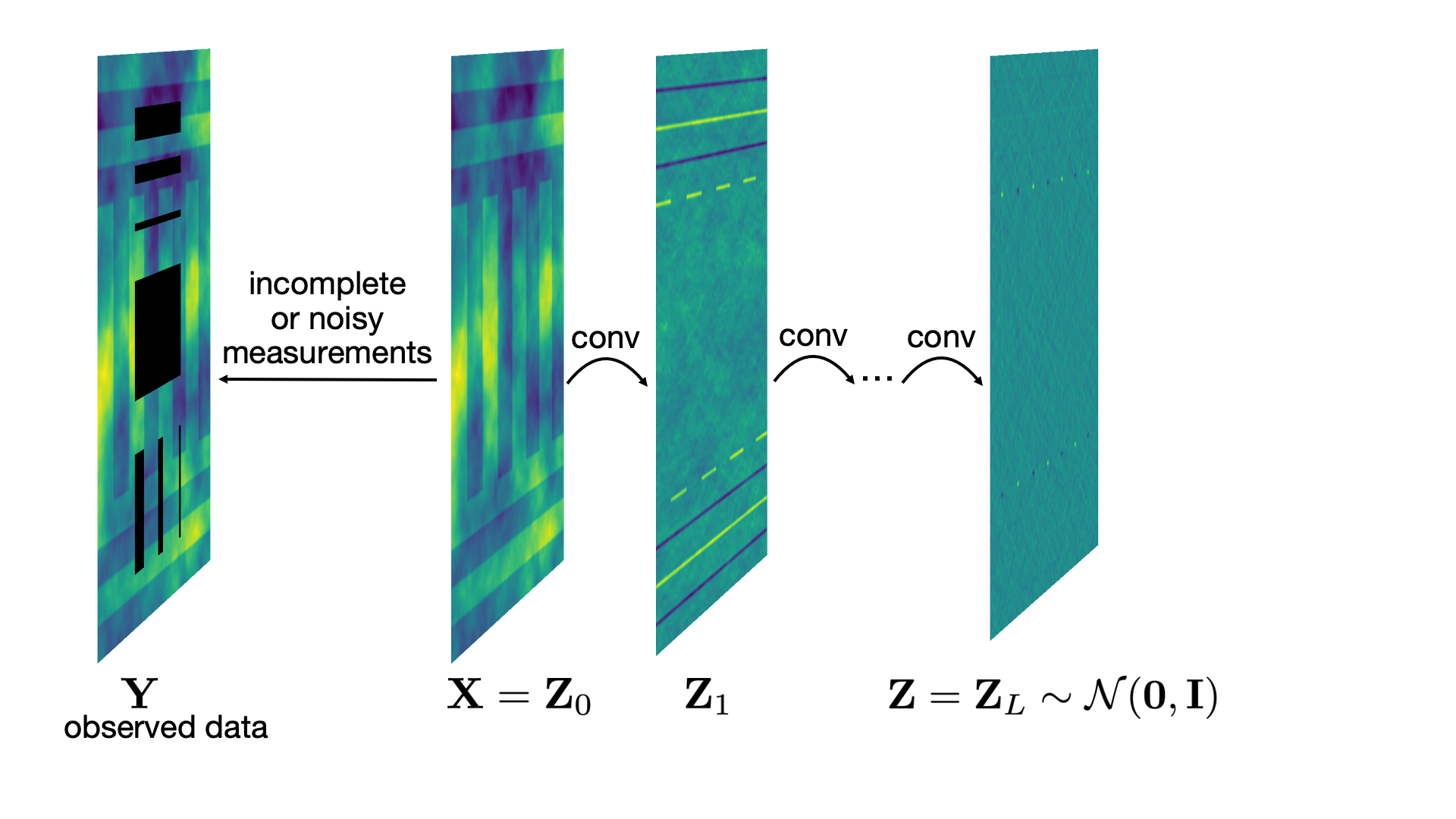}
\vspace{-5mm}
\caption{Illustration of the deep GMRF. The observed data $\mathbf{Y}$ are incomplete measurements of the latent DGMRF $\mathbf{X}$. The prior distribution of $\mathbf{X}$ is defined by a CNN that transforms the input $\mathbf{X}$ to the output image $\mathbf{Z}$ which has iid. standard normal pixel values.}\label{fig:plot_model}
\vspace{-3mm}
\end{center}
\end{figure}
The reason for defining $\mathbf{x}$ through an inverse transform $\mathbf{z} = \mathbf{g}(\mathbf{x})$, rather than a forward transform, is twofold. Firstly, it establishes a formal connection to traditional GMRFs, see Proposition~\ref{proposition-deep-GMRF-Matern} below. Secondly, it gives an AR model with infinite receptive field, meaning that the output prediction at each pixel depends on all the pixels, rather than just on nearby pixels, even for a one-layer model. Compared to dilated CNNs \citep[e.g.][]{Yu2015}, which achieve long (yet finite) range dependencies through several layers of dilated filters, this is a simpler construction.

\subsection{Computationally Cheap Convolutional Filters}\label{sec:conv-filters}
In this paper, we consider two special forms of convolutional filters: plus ($+$) filters and sequential (seq) filters, with forms
\begin{align}
&+:\;
\left[
\begin{smallmatrix}
 & a_3\\
a_2 & a_1 & a_4\\
 & a_5
\end{smallmatrix}
\right] 
&
& \text{seq}:\; \left[
\begin{smallmatrix}
 \vphantom{a_0} & &\\
 & a_1 & a_2\\
a_3 & a_4 & a_5
\end{smallmatrix}\right]
\label{eq:filters}
\end{align}
where $a_1,\ldots,a_5\in\mathbb{R}$ are parameters to be learned, and the empty positions are fixed to zero. The benefit of filters with these special designs is that they correspond to transforms for which $\det(\mathbf{G}_{l})$
in Eq.~(\ref{eq:detG}) can be cheaply computed.

Defining a linear DGMRF through a sequence of convolutions using different small filters is of course equivalent to defining it using a single convolution with a larger filter, apart from at the boundary. The equivalent larger filter can be obtained by sequentially convolving the smaller filters with themselves. The main motivation for using the deep definition is that it has cheap determinant calculations, using Eq.~(\ref{eq:detG}) and the $+$- and seq-filters, which is not the case for a general larger filter. Also, a deep definition results in fewer parameters to learn. For example, four full $3\times3$-filters has 36 parameters compared to 81 parameters for the corresponding $9\times9$-filter. Finally, the deep convolutional architecture has proven successful in countless applications of neural networks, which makes it plausible that also GMRFs should benefit from the same inductive biases. The addition of non-linearities between the layers is discussed in Section~\ref{sec:non-linear-deep-GMRFs}. We now discuss the two filter types and their determinant computation for the singlechannel case and discuss the multichannel construction in Section~\ref{subsec:multichannel}.
\subsubsection{$+$-Filters}
We begin with two propositions, the first connecting linear DGMRFs with $+$-filters to the GMRFs in Section~\ref{sec:Background}, and the second giving the cheap determinant computation.
\begin{proposition}\label{proposition-deep-GMRF-Matern}
The second-order intrinsic GMRF, as well as the GMRF approximation of a Matérn GP, are special cases of the linear DGMRF model with $+$-filters.
\end{proposition}
\begin{proof}
The non-zero pattern of the $+$-filter is the same as that of $\mathbf{w}_\mathbf{g}$ in Eq.~(\ref{eq:stencils}), meaning that a one-layer linear DGMRF with the same filter weights is equivalent to the second-order intrinsic GMRF. Similarly,  Eq.~(\ref{eq:discreteSPDE}) can be written as a linear DGMRF with $L$ layers of $+$-filters, for  $L \geq \gamma$. This requires that $\gamma$ of the layers have the same filter with $a_1=4+\kappa^2$ and $a_2=\cdots=a_5=-1$, and the other $L-\gamma$ layers to be identity functions.
\end{proof}
\begin{proposition}\label{proposition-+-filter-determinant}
The linear transform matrix $\mathbf{G}_+$ defined by singlechannel same convolution of an $H\times W$ image with the $+$-filter defined in Eq.~(\ref{eq:filters}), has determinant
\begin{align*}
    \det\left(\mathbf{G}_{+}\right)=\prod_{i=1}^{H}\prod_{j=1}^{W}&\left[a_{1}+2\sqrt{a_{3}a_{5}}\cos\left(\pi\frac{i}{H+1}\right)+\right.\\
    &\left.2\sqrt{a_{2}a_{4}}\cos\left(\pi\frac{j}{W+1}\right)\right],
\end{align*}
where $\sqrt{-1}$ is treated as imaginary. Computing the determinant thus has complexity $\mathcal{O}(N)$.
\end{proposition}
The proof, given in detail in the supplement, is to show that the factors of this product are indentical to the eigenvalues of $\mathbf{G}_{+}$,  which can be done by writing $\mathbf{G}_{+}$ as a Kronecker sum of tridiagonal Toeplitz matrices. Proposition~\ref{proposition-+-filter-determinant} provides a fast method for computing the determinant, which would have complexity $\mathcal{O}(N^3)$ for a general square matrix, or $\mathcal{O}(N^{3/2})$ if based on sparse Cholesky factorization \cite{Isham2004}. In practice, we reparameterize $a_1,\ldots,a_5$ to ensure that all the eigenvalues are real and positive, see details in the supplement. Without this constraint, we have observed unstable, oscillating solutions in some situations, and the constraint also ensures that the bijective assumption is valid.

\subsubsection{Seq-Filters}
For some ordering of the image pixels, the output pixels for a convolution with a seq-filter only depend on previous input pixels. This is equivalent to saying that the corresponding transform matrix $\mathbf{G}_{\text{seq}}$ yields a lower triangular matrix $\mathbf{P}^\top\mathbf{G}_{\text{seq}}\mathbf{P}$ for some permutation matrix $\mathbf{P}$, which implies that $\det(\mathbf{G}_{\text{seq}})=a_1^N$. Seq-filters are therefore extremely cheap, but somewhat restricted compared to $+$-filters, due to the inability to model non-sequential patterns. However, the seq-filter in Eq.~(\ref{eq:filters}) can be rotated/mirrored in eight different ways, each encoding a different pixel ordering.
Thus different sequential dependencies can be modelled in different layers. Also seq-filters can trivially be extended to $5\times5$- or $7\times7$-filters, still without changing the determinant. Connections to auto-regressive models, for example PixelCNN \citep{VandenOord2016}, are discussed in Section~\ref{sec:Related-Work}.

\subsection{Multichannel Construction}\label{subsec:multichannel}
When $C>1$, Eq.~(\ref{eq:layer-transition}) can be written on vector form as
\begin{align*}
\mathbf{z}_{l}&=\mathbf{G}_{l}\mathbf{z}_{l-1}+\mathbf{b}_{l}\\&=\left[\begin{array}{ccc}\mathbf{G}_{l,1,1} & \cdots & \mathbf{G}_{l,1,C}\\\vdots & \ddots & \vdots\\\mathbf{G}_{l,C,1} & \cdots & \mathbf{G}_{l,C,C}\end{array}\right]\mathbf{z}_{l-1}+\left[\begin{array}{c}b_{l,1}\mathbf{1}\\\vdots\\b_{l,C}\mathbf{1}\end{array}\right],
\end{align*}
where $\mathbf{G}_{l,i,j}$ is the transition matrix of a single convolution from input channel $j$ to output channel $i$, and $b_{l,j}$ is the bias of output channel $j$. In order to make $\det(\mathbf{G}_{l})$ computationally tractable, we set $\mathbf{G}_{l,i,j}=\mathbf{0}$ for $i<j$, making $\mathbf{G}_l$ lower block triangular and 
\begin{equation*}
\textstyle
\det\left(\mathbf{G}_{l}\right)=\prod_{c=1}^{C}\det\left(\mathbf{G}_{l,c,c}\right).
\end{equation*}
The ordering of channels could vary between different layers to allow information to flow back and forth between all channels. One could also add invertible $1\times1$ convolutions \citep{Kingma2018a} between layers for this ordering to be dynamic and learnable. A multichannel DGMRF could learn more interesting representations by storing different features in different channels of the hidden layers. Even with singlechannel data, the hidden layers can be multichannel by using a multiscale architecture \citep{Dinh2017}.

\subsection{Non-Linear Extension}\label{sec:non-linear-deep-GMRFs}
The linear DGMRFs can be extended by adding non-linear activation functions between layers of the neural network $\mathbf{g}_{\boldsymbol{\theta}}(\mathbf{x})$. Formally, Eq.~(\ref{eq:layer-transition}) is replaced by
\begin{equation*}
\mathbf{Z}_{l}=\psi_l\left(\text{conv}\left(\mathbf{Z}_{l-1},\mathbf{w}_{l}\right)+\mathbf{b}_{l}\right),
\label{eq:non-linear-layer-transition}
\end{equation*}
where $\psi_{l}$ is a non-linear scalar function that operates element-wise. We restrict $\psi_{l}$ to be strictly increasing, to ensure that $\mathbf{g}_{\boldsymbol{\theta}}(\mathbf{x})$ is a bijection.

The distribution of $\mathbf{x}$ can now be computed by the change of variable rule
\begin{align*}
\begin{split}
\,&\log p(\mathbf{x})=\log p\left(\mathbf{z}\right)+\log\left|\det\left(d\mathbf{z}/d\mathbf{x}\right)\right|\\&=\log p\left(\mathbf{z}\right)+\sum_{l=1}^{L}\log\left|\det\left(\mathbf{G}_{l}\right)\right|+\sum_{l=1}^{L}\sum_{i=1}^{N}\log\left|\psi_{l}^{\prime}\left(h_{l,i}\right)\right|,
\label{eq:z=changeOfVariable}
\end{split}
\end{align*}
where 
$\mathbf{h}_l=\text{vec}(\text{conv}(\mathbf{Z}_{l-1},\mathbf{w}_{l})+\mathbf{b}_{l})$ is the input to $\psi_{l}$ and
$\psi_{l}^{\prime}$ is the derivative. Just as in the linear case, the computational cost of this density is linear.

Per default we assume that $\psi_{l}$ are Parametric Rectified Linear Units (PReLUs), defined by
\begin{equation*}
\psi_{l}\left(h\right)=\begin{cases}
\begin{array}{c}
h,\\
\alpha_{l}h,
\end{array} & \begin{array}{c}
\text{for}\,\,\,h\geq0\\
\text{for}\,\,\,h<0
\end{array}\end{cases},
\end{equation*}
where $\alpha_{l}$ are learnable parameters with $\alpha_{l}>0$.
We can now add $\alpha_1,\ldots,\alpha_L$ to the parameters $\boldsymbol{\theta}$ and optimize them jointly.

\section{Learning and Inference}\label{sec:Inference}
There exist two kinds of unknown variables that we wish to infer: the latent field $\mathbf{x}$ and the model parameters $\boldsymbol{\theta}$. Since we are not interested in the posterior uncertainty of $\boldsymbol{\theta}$ directly, we take a practical course of action and optimize these, using a variational lower bound on the marginal likelihood $p(\mathbf{y}|\boldsymbol{\theta})$. Given the optimal value $\hat{\boldsymbol{\theta}}$, for the linear model, we make a fully Bayesian treatment of the posterior $p(\mathbf{x}|\hat{\boldsymbol{\theta}},\mathbf{y})$.
\subsection{Optimization of Parameters}
For the linear DGMRF, the marginal likelihood can be written on closed form as
\begin{equation*}
p\left(\mathbf{y}|\boldsymbol{\theta}\right)=\left.\frac{p\left(\mathbf{y}|\mathbf{x},\boldsymbol{\theta}\right)p\left(\mathbf{x}|\boldsymbol{\theta}\right)}{p\left(\mathbf{x}|\mathbf{y},\boldsymbol{\theta}\right)}\right|_{\mathbf{x}=\mathbf{x}^{*}}
\end{equation*}
for arbitrary value of $\mathbf{x}^{*}$. Unfortunately, this expression requires the determinant of the posterior precision matrix $\det(\mathbf{G}^{\top}\mathbf{G}+\sigma^{-2}\mathbf{I}_\mathbf{m})$, which is computationally infeasible for large $N$. Instead, we focus on the variational evidence lower bound (ELBO) $\mathcal{L}\left(\boldsymbol{\theta},\boldsymbol{\phi},\mathbf{y}\right)\leq \log p(\mathbf{y}|\boldsymbol{\theta})$ which can be written as
\begin{equation*}
\mathcal{L}\left(\boldsymbol{\theta},\boldsymbol{\phi},\mathbf{y}\right)=\mathbb{E}_{q_{\boldsymbol{\phi}}(\mathbf{x})}\left[-\log q_{\phi}(\mathbf{x})+\log p\left(\mathbf{y},\mathbf{x}|\boldsymbol{\theta}\right)\right],
\end{equation*}
where $q_{\boldsymbol{\phi}}(\mathbf{x})$ is the variational posterior approximation, which depends on variational parameters $\boldsymbol{\phi}$. We here only intend to use $q_{\boldsymbol{\phi}}(\mathbf{x})$ as a means for optimizing $\boldsymbol{\theta}$, and not for example to make posterior predictions. For simplicity, we choose $q_{\boldsymbol{\phi}}(\mathbf{x})=\mathcal{N}(\mathbf{x}|\boldsymbol{\nu},\mathbf{S})$ with diagonal covariance matrix, and $\boldsymbol{\phi}=\{\boldsymbol{\nu},\mathbf{S}\}$. After inserting the variational and model densities and simplifying, we can write the ELBO as
\begin{align}
\begin{split}
\label{eq:ELBO}
&\mathcal{L}\left(\boldsymbol{\theta},\boldsymbol{\phi},\mathbf{y}\right)=\frac{1}{2}\log\left|\det\left(\mathbf{S}_{\boldsymbol{\phi}}\right)\right|-M\log\sigma_{\boldsymbol{\theta}}+\log\left|\det\left(\mathbf{G}_{\boldsymbol{\theta}}\right)\right|\\&-\frac{1}{2}\mathbb{E}_{q_{\boldsymbol{\phi}}(\mathbf{x})}\bigg[\mathbf{g}_{\boldsymbol{\theta}}(\mathbf{x})^{\top}\mathbf{g}_{\boldsymbol{\theta}}(\mathbf{x})
+\frac{1}{\sigma_{\boldsymbol{\theta}}^{2}}\left(\mathbf{y}-\mathbf{x}\right)^{\top}\mathbf{I}_\mathbf{m}\left(\mathbf{y}-\mathbf{x}\right)\bigg],
\end{split}
\end{align}
where constant terms have been omitted, and all parameters have been subscripted with $\boldsymbol{\theta}$ and $\boldsymbol{\phi}$ to clarify whether they are model or variational parameters. Additionally, we use the reparameterization trick \citep{Kingma2013} replacing the last expectation with a sum over $N_q$ standard random samples $\boldsymbol{\varepsilon}_{1},\ldots,\boldsymbol{\varepsilon}_{N_q}$, and set $\mathbf{x}_i=\boldsymbol{\nu}_{\boldsymbol{\phi}}+\mathbf{S}_{\boldsymbol{\phi}}^{1/2}\boldsymbol{\varepsilon}_{i}$ in the sum. This gives an unbiased estimator of the ELBO, which has low variance. Moreover, this estimator is differentiable with respect to $\boldsymbol{\theta}$ and $\boldsymbol{\phi}$ and can be used for stochastic gradient optimization with autodiff and backprop. 
Parameter learning will be fast, with a time complexity that is $\mathcal{O}(N)$ for a fixed number of iterations of optimization, since backprop in a CNN is linear and so are the determinant computations described in Section~\ref{sec:conv-filters}. This can be compared with $\mathcal{O}(N^3)$ for standard GP and $\mathcal{O}(N^{3/2})$ for standard GMRF inference in 2D problems, based on the Cholesky-decomposition. The ELBO can be trivially extended to the non-linear model by replacing $\log|\det(\mathbf{G}_{\boldsymbol{\theta}})|$ in Eq.~(\ref{eq:ELBO}).
\subsection{Exact Inference for the Latent Field}
For a linear DGMRF, the conditional posterior $p(\mathbf{x}|\hat{\boldsymbol{\theta}},\mathbf{y})$ is also a GMRF, see Eq.~(\ref{eq:posteriorGMRF}). Even though computing this density is too costly in general, it is possible to compute the posterior mean and to draw samples from the posterior, which can be used for making predictions. Both require solving linear equation systems $\tilde{\mathbf{Q}}\mathbf{x}=\mathbf{c}$ involving the posterior precision matrix $\tilde{\mathbf{Q}}=\mathbf{G}^{\top}\mathbf{G}+\sigma^{-2}\mathbf{I}_\mathbf{m}$. For this we use the conjugate gradient (CG) method (see e.g. \citealp{barrett1994templates}), which is an iterative method that, rather than exactly computing $\mathbf{x}=\tilde{\mathbf{Q}}^{-1}\mathbf{c}$, iteratively minimizes the relative residual $\Vert\tilde{\mathbf{Q}}\mathbf{x}-\mathbf{c}\Vert/\Vert\mathbf{c}\Vert$ until it falls below some threshold $\delta$, that can be set arbitrarily low. In practice, $\delta=10^{-7}$ gives solutions that visually appear to be identical to the exact solution. CG only requires matrix-vector-multiplications, which means that the multiplications with $\mathbf{G}$ and $\mathbf{G}^\top$ can be implemented in a matrix-free fashion using convolution.

Posterior sampling for $\mathbf{x}$ can be performed using the method of \citet{Papandreou2010} by computing
\begin{equation*}
\mathbf{x}_{s}=\tilde{\mathbf{Q}}^{-1}\left(\mathbf{G}^{\top}\left(\mathbf{u}_{1}-\mathbf{b}\right)+\frac{1}{\sigma^{2}}\left(\mathbf{y}+\sigma\mathbf{I}_\mathbf{m}\mathbf{u}_{2}\right)\right),
\end{equation*}
where $\mathbf{u}_1$ and $\mathbf{u}_2$ are standard Gaussian random vectors. It can be easily verified that $\mathbf{x}_{s}$ is Gaussian, with the same mean and covariance as the posterior.

 \begin{algorithm}[tb]
   \caption{Inference algorithm}
   \label{alg:inference}
\begin{algorithmic}
   \STATE {\bfseries Input:} data $\mathbf{y}$, model structure, learning rate, $N_q$, etc.
   \vspace{1mm}
   \STATE Initialize values for the model parameters\newline $\boldsymbol{\theta}\subseteq\{(\mathbf{w}_{l},\mathbf{b}_{l},\alpha_{l})_{l=1,\ldots,L},\sigma^2\}$ and the variational parameters $\boldsymbol{\phi}=\{\boldsymbol{\nu},\mathbf{S}\}$
   \vspace{1mm}
   \STATE Optimize the ELBO (Eq.~\ref{eq:ELBO}) wrt. $\boldsymbol{\theta}$ and $\boldsymbol{\phi}$ using Adam to obtain $\hat{\boldsymbol{\theta}}$ and $\hat{\boldsymbol{\phi}}$ 
   \vspace{1mm}
   \IF{the model is linear}
   \STATE Compute the mean and marginal variances of the posterior $p(\mathbf{x}|\mathbf{y},\hat{\boldsymbol{\theta})}$ using CG and simple RBMC
   \ELSE
   \STATE Approximate the mean and marginal variances of the posterior $p(\mathbf{x}|\mathbf{y},\hat{\boldsymbol{\theta}})$ with $\hat{\boldsymbol{\nu}}$ and $\text{diag}(\hat{\mathbf{S}})$
   \ENDIF
   \vspace{1mm}
   \STATE Compute the predictive means $\mathbb{E}(y_{i}^{*}|\mathbf{y},\hat{\boldsymbol{\theta}})=\mathbb{E}(x_{i}|\mathbf{y},\hat{\boldsymbol{\theta}})$ and variances $\text{Var}(y_{i}^{*}|\mathbf{y},\hat{\boldsymbol{\theta}})=\text{Var}(x_{i}|\mathbf{y},\hat{\boldsymbol{\theta}})+\sigma^{2}$
\end{algorithmic}
\end{algorithm}
Given a number of posterior samples, the posterior marginal variances $\text{Var}(x_{i}|\mathbf{y},\hat{\boldsymbol{\theta}})$ can be naively approximated using Monte Carlo estimation, but more efficiently approximated using the simple Rao-Blackwellized Monte Carlo (simple RBMC) method by \citet{Siden2018}. This provides a way to compute the posterior predictive uncertainty as
\begin{equation*}
    \text{Var}(y_{i}^{*}|\mathbf{y},\hat{\boldsymbol{\theta}})=\text{Var}(x_{i}|\mathbf{y},\hat{\boldsymbol{\theta}})+\sigma^{2}.
\end{equation*}
The time complexity is here mainly decided by the complexity of the CG method, which can be described as $\mathcal{O}(N\sqrt{\kappa})$, where $\kappa$ is the condition number of $\tilde{\mathbf{Q}}$ \citep{Shewchuk1994}. It is difficult to say how $\kappa$ depends on $N$ for a general $\tilde{\mathbf{Q}}$, but for a two-dimensional, second-order elliptic boundary value problems, for example the SPDE approach with $\gamma=1$, $\kappa$ is linear in $N$ so the method is $\mathcal{O}(N^{3/2})$. Even though this is the same as for the sparse Cholesky method, we have observed CG to be a lot faster in practice. Also, the storage requirements for CG is $\mathcal{O}(N)$, versus $\mathcal{O}(N \log N)$ for Cholesky. 

A drawback with non-linear DGMRFs is that they do not give a simple posterior $p(\mathbf{x}|\hat{\boldsymbol{\theta}},\mathbf{y})$. We could use the variational approximation $q_{\boldsymbol{\phi}}(\mathbf{x})$ instead, but the proposed Gaussian independent variational family is probably too simple to approximate the true posterior in a satisfying way. We discuss possible solutions to this limitation in Section~\ref{sec:Conclusions}.

The inference algorithm is summarized in Algorithm~\ref{alg:inference}.
\vfill

\section{Related Work}\label{sec:Related-Work}
\subsection{GMRFs}
GMRFs have a long history in spatial and image analysis (see e.g. \citealp{Woods1972}; \citealp{Besag1974} and \citealp{Isham2004}). The mentioned SPDE approach has been extended in numerous ways leading to flexible GMRF models, with for example non-stationarity  \citep{FuglstadLindgren2015}, oscillation \citep{Lindgren2011} and non-Gaussianity \citep{bolin2014spatial}. A nested (deep) SPDE is introduced by \citet{Bolin2011}, which results in a model similar to ours, but which lacks the connection to CNNs and some of the computational advantages that we provide.

\citet{Papandreou2010} use GMRFs for inpainting in natural images, in combination with CG to obtain posterior samples. They do not learn the form of the prior precision matrix, but instead rewrite the model as a product of Gaussian experts, and learn the mean and variance of the Gaussian factors. \citet{Papandreou2011} assume a heavy-tailed non-Gaussian prior for the latent pixel values, and use a variational GMRF approximation for the posterior.

\subsection{GPs}
GPs \citep{stein1999,williams2006gaussian} are commonplace for modelling spatially dependent data, not necessarily restricted to observations on the grid. 
The GP covariance kernel, which encodes the spatial dependencies, can be made flexible \citep{Wilson2013} and deep \citep{Dunlop2018,roininen2019hyperpriors}. However, standard GPs are limited by $\mathcal{O}(N^3)$ computational complexity (assuming the number of measurements $M\sim\mathcal{O}(N)$).

Inducing point methods \citep{quinonero2005unifying,snelson2006sparse} 
can reduce the complexity to $\mathcal{O}(P^2N+P^3)$ or even $\mathcal{O}(P^3)$ \citep{Hensman2013}, where $P$ is the number of inducing points. However, 
for grid data, these methods tend to over-smooth the data \citep{wilson2015kernel} unless $P$ is chosen in the same order of magnitude as $N$. When the data are on a regular grid and fully observed, 
Kronecker and Toeplitz methods can be used for fast computation (e.g. \citealp{Saatci2011}; \citealp{wilson2014covariance}). Under certain assumptions about the interactions across the input dimensions, such as additivity and separability, this can reduce the complexity to $\mathcal{O}(N\log N)$. These methods can also be extended to when the data are not on the grid \citep{wilson2015kernel}, or when the grid is incomplete \citep{Stroud2016}.

\subsection{Deep Generative Models}\label{sec:DGMs}
Deep generative models 
are normally trained on large datasets, and can then be used to generate new samples from the data distribution, or to explore learned latent representations. For example, generative adversarial networks (GANs, \citealp{goodfellow2014}) have been used for inpainting problems in natural images (see e.g. \citealp{yu2018}; \citealp{liu2018}) with impressive results. Our model is mainly targeted at smaller datasets, as single images, where too complex models are likely to overfit. Instead, we achieve an infinite receptive field through the inverse definition $\mathbf{z} = \mathbf{g}(\mathbf{x})$, which require few parameters. Also, linear DGMRFs
can readily provide uncertainty estimates, which is prohibitive for most deep generative models. Overall, deep generative models are difficult to use for typical spatial problems, but these models share ideas with our proposed model as we discuss below. 

Flow-based or invertible generative models, e.g., NICE \citep{dinh2014nice}, 
normalizing flows \citep{rezende15}, IAF \citep{kingma2016improved}, MAF \citep{papamakarios2017masked}, real NVP \citep{Dinh2017}, Glow \citep{Kingma2018a}
and i-ResNet \citep{pmlr-v97-behrmann19a}, model $\mathbf{x}$ as a bijective function of a latent variable from a simple base distribution,
just as our model. 
The likelihood can be optimized directly for parameter learning, but generally require all pixels to be non-missing, that is, these methods are not designed to handle the situation with incomplete or noisy observations that we are considering. One connection to our work is that a linear DGMRF with seq-filters can be seen as a fully linear MAF with masked convolutional MADE-layers. \citet{dinh2014nice} use a learned model for inpainting with projected gradient ascent, which gives the posterior mode but no uncertainty quantification.
\citet{Lu2019} provide a conditional version of Glow for inpainting, but this requires missing pixels to be known during training, and the predictive distribution is modeled directly and not through a Bayesian inversion of the model.

Auto-regressive models, e.g. PixelRNN \citep{VandenOord2016a}, PixelCNN \citep{VandenOord2016} and PixelCNN++ \citep{Salimans2017}, model the pixel values sequentially
similarly to our proposed seq-filters, but there are some major differences. The seq-filters use masked convolutions to obtain cheap determinant computations 
and the ordering of pixels can change between layers, whereas auto-regressive models have no latent variables and are trained using the same image as input and output.
\citep{VandenOord2016a} consider image completion, but are limited to the e.g. the case where the bottom half of the image is missing, and can't handle a general mask.

Variational autoencoders \citep{Kingma2013} differ from our model in that they use the deep transform of the latent variables for modeling the parameters of the distribution of the data, rather than for the data directly. This makes the recovery of the latent variables intractable, which is not the case for our model, however, we still use the same variational optimization algorithm for speedups.

Deep image prior (DIP, \citealp{ulyanov2018deep}) uses CNNs, trained on single images, for denoising and inpainting, with promising results. However, this is not truly a probabilistic generative model, since the latent variables are not inferred, but fixed to some initial random values. Thus, it would be difficult to output predictive uncertainty from such a model.
\begin{figure*}[ht]
\begin{center}
\includegraphics[width=175mm,trim={2mm 2mm 0 1mm},clip]{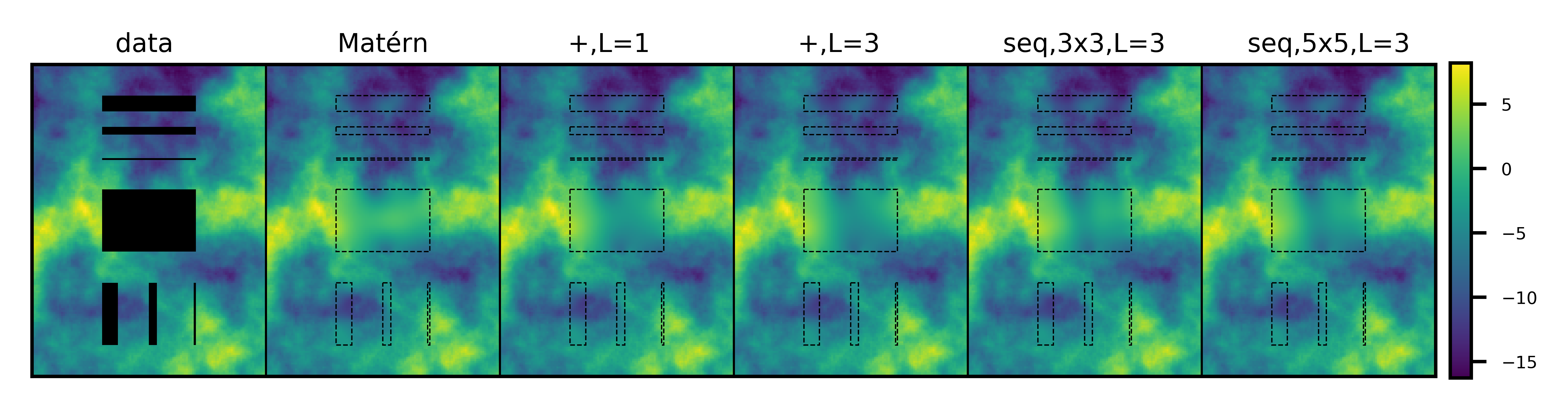}
\includegraphics[width=175mm,trim={2mm 2mm 0 2mm},clip]{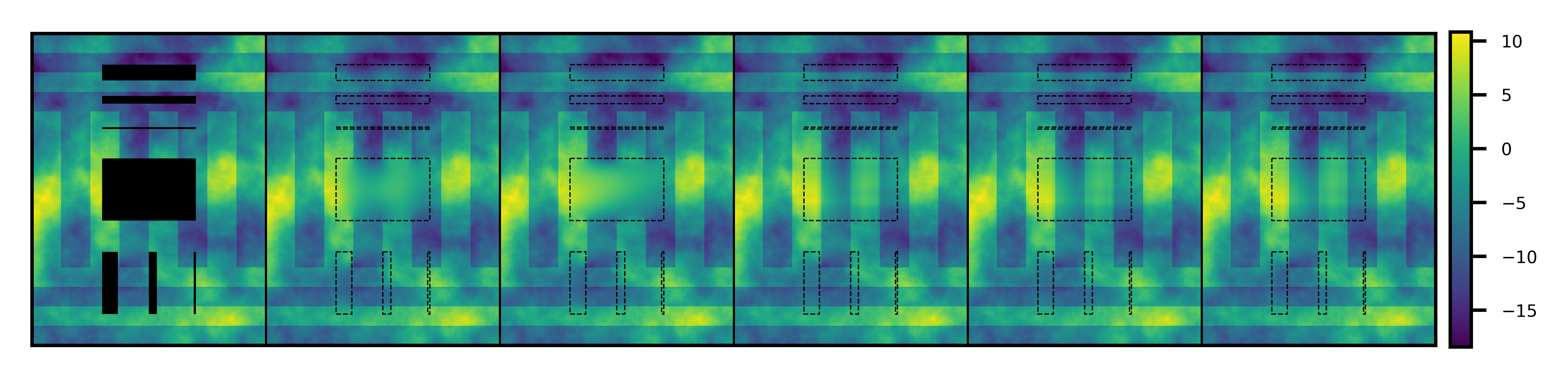}
\vspace{-0.5cm}
\caption{Posterior mean for inpainting the 160$\times$120 pixel toy data without edges (top) and with edges (bottom). The second column shows inpainting by the same Matérn GMRF model and hyperparameters that were used to generate the data without edges.}
\label{fig:toy-data}
\end{center}
\end{figure*}
\section{Experiments}\label{sec:Results}
We have implemented DGMRF
in TensorFlow \citep{abadi2016tensorflow}, taking advantage of autodiff and GPU computations. We train the parameters, and compute the predictive distribution using the CG algorithm. To avoid boundary effects, the images are extended with a 10-pixel wide frame of missing values at all sides. Additional details about the implementation can be found in the supplement. Code for our methods and experiments are available at \url{https://bitbucket.org/psiden/deepgmrf}.
\subsection{Toy Data}
We demonstrate the behaviour of our model for an inpainting problem on the two toy datasets in Figure~\ref{fig:toy-data}, which have size 160$\times$120 pixels. The data in the first row are generated as a random sample from the Matérn GMRF in Eq.~(\ref{eq:discreteSPDE}) with $\gamma=1$, $\tau=1$, and $\kappa^2=8/50^2$ corresponding to a spatial correlation range of 50 pixels. The data in the second row are the same, but we have added horizontal and vertical edges, to investigate the model's ability to capture such patterns. Column 2 shows the posterior mean for $\mathbf{x}$ when the Matérn model with the same hyperparameters is used for prediction, which gives optimal results for the first row, but which over-smooths the edges in the second row. The corresponding results for different instances of the linear DGMRF, are shown in column 3-6. They all perform well in the simple case without edges, in which the $+$-filter models contain the true model according to Proposition~\ref{proposition-deep-GMRF-Matern}. In the case with edges, the model with depth $L=1$ which corresponds to a standard GMRF is too simple to handle the more complex structure in the data, but for $L=3$, all the models give reasonable results, that preserve the edges.

Figure~\ref{fig:intermediate-layers} displays the learned $+$-filters of the 3-layer model, and the values of valid pixels of the hidden layers, when original data, including missing pixels, are used as input. The first two filters learns differentiation in the vertical and horizontal direction, and the third filter is close to an identity function. Most spatial structure is removed in layer 3, that is assumed to be standard normal by the model.
\begin{figure}[ht]
\begin{center}
\includegraphics[width=60mm]{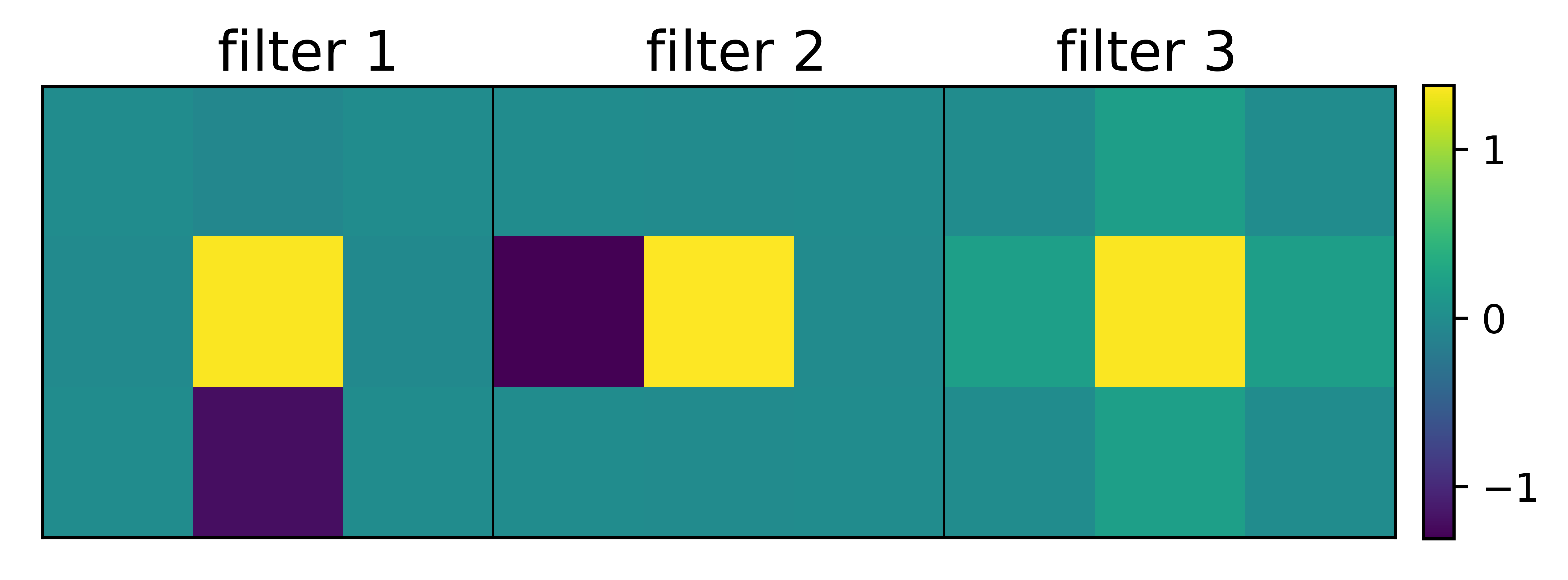}
\includegraphics[width=85mm]{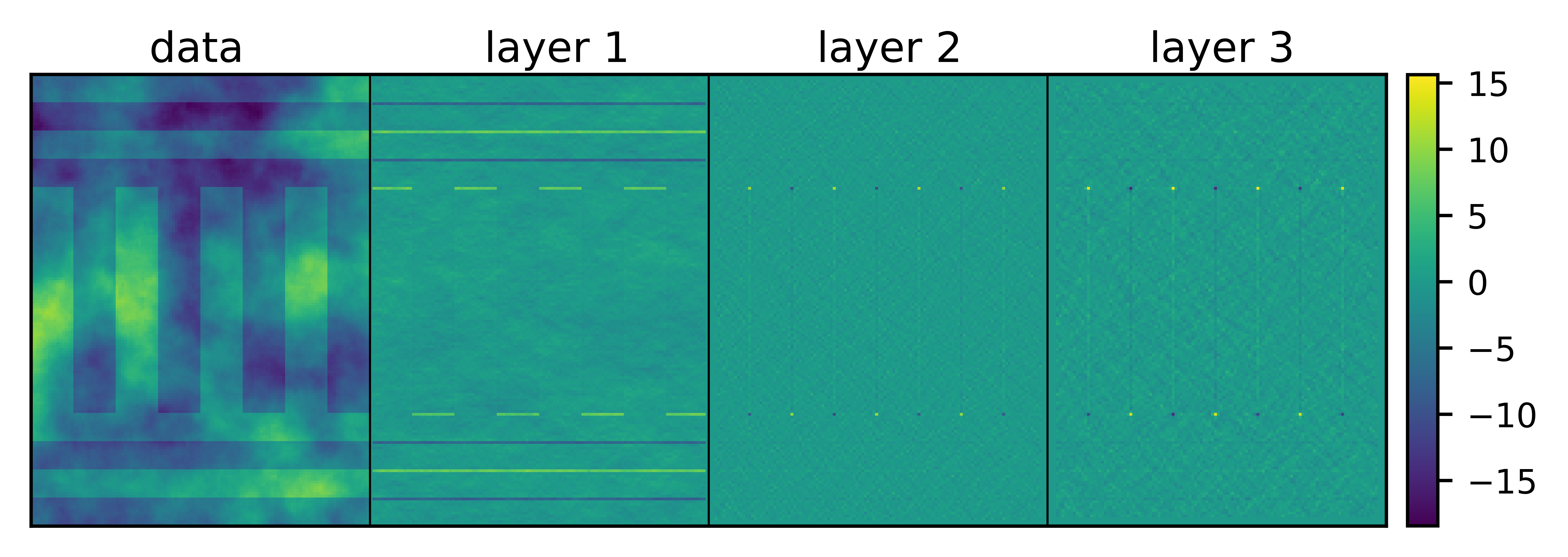}
\vspace{-.5cm}
\caption{A linear deep GMRF with $3$ layers of learned $+$-filters for the toy data with edges. The filters collaborate to remove spatial structures to the final layer.}\label{fig:intermediate-layers}
\vspace{-.4cm}
\end{center}
\end{figure}
\subsection{Satellite Data}
We compare our method against some popular methods for large data sets in spatial statistics, by considering the satellite data of daytime land surface temperatures, used in the competition by \citet{Heaton2018}. The data are on a 500$\times$300 grid, with 105,569 non-missing observations as training set. The test set consists of 42,740 observations and have been selected as the pixels that were missing due to cloud cover on a different date. The dataset and the missing pixels are shown in the supplement, together with the posterior output from our model. The data and code for some of the methods can be found at \url{https://github.com/finnlindgren/heatoncomparison}. The participants of the competition were instructed to use exponential correlation if applicable, and to add a constant mean and linear effects of the latitude and longitude. For this reason, we extend the measurement equation to include linear trends, so that
\begin{equation}\label{eq:linear-trend}
    y_{i}|x_{i}\sim\mathcal{N}\left(y_{i}|x_{i}+\mathbf{F}_{i,\cdot}\boldsymbol{\beta},\sigma^{2}\right),
\end{equation}
where $\mathbf{F}$ is a spatial covariate matrix with columns corresponding to (constant, longitude, latitude), and $\boldsymbol{\beta}$ is a 3-dimensional regression coefficient vector, which can be integrated out jointly with $\mathbf{x}$ for the predictions, see details in the supplement.

Table~\ref{satellite-table} compares different instances of our model with the methods in the competition, which are briefly described in the supplement, and in more detail in \citet{Heaton2018}. One of the top competitors, SPDE, is essentially the same method as described in Section~\ref{sec:SPDE}. For comparison with another deep method, we have also included results for DIP using the authors own implementation\footnote{https://github.com/DmitryUlyanov/deep-image-prior}. For DIP, removal of linear trends and normalization to $[0,1]$ were done in preprocessing, and these steps were inverted before evaluation. As DIP does not give uncertainty estimates, we also compare with an ensemble of 10 DIP models, using the ensemble mean and standard deviation as predictive distribution. The scores used are mean absolute error (MAE), root-mean-squared-error (RMSE), mean continuous rank probability score (CRPS), mean interval score (INT), and prediction interval coverage (CVG). CRPS and INT are proper scoring rules, that also account for the predictive uncertainty \citep{gneiting2007strictly}. These are the same scores that were used in the competition. 

\begin{table}[ht]
\caption{Prediction scores on the satellite data. The scores of the methods in the upper pane come from Table 3 in \citet{Heaton2018}. Our models are presented in the lower pane. Lower scores are better, except from CVG, for which 0.95 is optimal. The results for our models are averages computed across five random seeds. Standard deviations across seeds are shown in parenthesis for $\text{seq}_{5\times5,L=5}$, and for the other models in the supplement.}
\abovespace
\label{satellite-table}
\centering
\small
\begin{tabular}{rlllll}
\toprule
Method & MAE & RMSE & CRPS & INT & CVG \\
\midrule
  FRK & 1.96 & 2.44 & 1.44 & 14.08 & 0.79 \\ 
  Gapfill & 1.33 & 1.86 & 1.17 & 34.78 & 0.36 \\ 
  LatticeKrig & 1.22 & 1.68 & 0.87 & 7.55 & 0.96 \\ 
  LAGP & 1.65 & 2.08 & 1.17 & 10.81 & 0.83 \\ 
  MetaKriging & 2.08 & 2.50 & 1.44 & 10.77 & 0.89 \\ 
  MRA & 1.33 & 1.85 & 0.94 & 8.00 & 0.92 \\ 
  NNGP & 1.21 & 1.64 & 0.85 & 7.57 & \textbf{0.95} \\ 
  Partition & 1.41 & 1.80 & 1.02 & 10.49 & 0.86 \\ 
  Pred. Proc. & 2.15 & 2.64 & 1.55 & 15.51 & 0.83 \\ 
  SPDE & 1.10 & 1.53 & 0.83 & 8.85 & 0.97 \\ 
  Tapering & 1.87 & 2.45 & 1.32 & 10.31 & 0.93 \\ 
  Peri. Embe. & 1.29 & 1.79 & 0.91 & 7.44 & 0.93 \\ 
\midrule
  DIP & 1.53 & 2.06 & - & - & - \\
  DIP ensemble & 1.30 & 1.67 & 0.96 & 11.82 & 0.72 \\
\midrule
  DGMRF (our) & & & & & \\ 
  $\text{seq}_{5\times5,L=1}$ & 1.06 & 1.42 & 0.76 & 7.21 & 0.97 \\ 
  $\text{seq}_{5\times5,L=3}$ & 0.95 & 1.3 & 0.75 & 8.29 & 0.97 \\ 
  $\text{seq}_{5\times5,L=5}$ & \textbf{0.93} & \textbf{1.25} & \textbf{0.74} & 8.14 & 0.97 \\ 
   & (.037) & (.051) & (.012) & (.461) & (.001) \\ 
  $\text{seq}_{3\times3,L=5}$ & 1.16 & 1.57 & 0.81 & \textbf{6.98} & 0.97 \\
  $+_{L=5}$ & 1.09 & 1.47 & 0.78 & 7.63 & 0.97 \\ 
  $\text{seq}_{5\times5,L=5,\text{NL}}$ & 1.37 & 1.87 & - & - & - \\
\bottomrule
\end{tabular}
\vspace{-.4cm}
\end{table}
For the DGMRFs, based on the first three scores, the seq-filters of size 5$\times$5 perform better compared to those of size 3$\times$3 and the $+$-filters, which may be due to the increased flexibility of larger filters. Moreover, deeper models tend to give better results. For the non-linear (NL) model, CG cannot be used to compute posterior mean and uncertainty, and the first two scores are instead computed based on the mean of the variational approximation. We note that the NL model performs worse than the linear. Our primary explanation for this disappointing result is that the variational posterior is insufficient in approximating the true posterior. An independent distribution naturally has limited ability to approximate a spatially dependent posterior, and empirically we have also seen that the variational approximation for the linear model performs much worse than the true posterior (results not shown). DGMRF outperforms all the methods from the competition on all criteria except CVG, where it is slightly worse than NNGP. In terms of MAE, RMSE and CRPS the improvement is substantial. It is difficult to compare computing times due to different hardware, but our method takes roughly 2.5h for the $\text{seq}_{5\times5,L=5}$ model using a Tesla K40 GPU, out of which roughly 95\% of the time is for parameter learning.

\section{Conclusions and Future Work}\label{sec:Conclusions}
We have proposed deep GMRFs which enable us to view (high-order) GMRF models as CNNs. We have focused our attention on lattice-based graphs to clearly show the connection to conventional CNNs, however, the proposed method can likely be generalized to arbitrary graphs via graph convolutions (see, e.g., \citealp{xu2018how}), and to continuously referenced data, similar to \citet{Bolin2011}. We have also primarily considered \emph{linear} CNNs, resulting in a model that is indeed equivalent to a GMRF. The DGMRF nevertheless has favorable properties compared to conventional algorithms for GMRFs: the CNN architecture opens up for simple and efficient learning, even though the corresponding graph will have a high connectivity when using multiple layers of the CNN. Empirically we have found that multi-layer architectures are indeed very useful even in the linear case, enabling the model to capture complex data dependencies such as distinct edges, and result in state-of-the-art performance on a spatial benchmark problem.

Using a CNN for the inverse mapping when defining the DGMRF results in a spatial AR model. This gives rise to an infinite receptive field, but can also result in instability of the learned prior. We have constrained the filter parameters to yield real eigenvalues, and under this constraint we have not seen any issues with instability. A more principled way of addressing this potential issue is left for future work.

The CNN-based interpretation offers a natural extension, namely to introduce non-linearities to enable modeling of more complex distributions for the prior $p(\mathbf{x})$. A related extension would be to replace $p(\mathbf{x})$ with one of the flow-based or auto-regressive generative models mentioned in Section~\ref{sec:DGMs}, but appropriately restricted to avoid overfitting on a single image. Further exploring these extensions requires future work. In particular we believe that the independent variational approximation is insufficient to accurately approximate the posterior distribution over latent variables in non-linear DGMRFs. 
One way to address this limitation is to parameterize the square-root of the covariance matrix $\mathbf{S}_{\boldsymbol{\phi}}^{1/2}$ of the variational approximation $q_{\boldsymbol{\phi}}$ directly as a lower triangular matrix. Another approach is to model $q_{\boldsymbol{\phi}}$ as a GMRF with the same graph structure as the original model.





\bibliographystyle{icml2020}
\bibliography{paper}
\newpage
\twocolumn[
\section*{\centering\Large Supplementary material}
\vspace{5mm}
]
\section*{Proof of Proposition~\ref{proposition-+-filter-determinant}}
The transform matrix $\mathbf{G}_{+}$ can be written as
\begin{equation*}
\mathbf{G}_{+}=\mathbf{T}_{1}\oplus\mathbf{T}_{2}=\mathbf{T}_{1}\otimes\mathbf{I}_{H}+\mathbf{I}_{W}\otimes\mathbf{T}_{2},
\end{equation*}
where $\oplus$ denotes the Kronecker sum and $\otimes$ the Kronecker product, $\mathbf{T}_{1}$ is a $W\times W$ tridiagonal Toeplitz matrix, denoted $\mathbf{T}_{1}=\left(W;a_{2},a_{1}/2,a_{4}\right)$, meaning that
\begin{equation*}
    \mathbf{T}_{1}=\left[\begin{array}{ccccc}
a_{1}/2 & a_{4} &  &  & 0\\
a_{2} & a_{1}/2 & a_{4}\\
 & a_{2} & \ddots & \ddots\\
 &  & \ddots &  & a_{4}\\
0 &  &  & a_{2} & a_{1}/2
\end{array}\right],
\end{equation*}
and similarly $\mathbf{T}_{2}=\left(H;a_{3},a_{1}/2,a_{5}\right)$. For the eigenvalues of a Kronecker sum, it holds that if $\lambda_1$ is an eigenvalue of $\mathbf{T}_{1}$ and $\lambda_2$ is an eigenvalue of $\mathbf{T}_{2}$, then $\lambda_1+\lambda_2$ is an eigenvalue of $\mathbf{T}_{1}\oplus\mathbf{T}_{2}$ \citeS{graham1981kronecker}. Moreover, the eigenvalues of a tridiagonal Toeplitz matrix $\mathbf{T}=\left(n;b,a,c\right)$ have a simple formula 
\begin{equation*}
    \lambda_{i}\left(\mathbf{T}\right)=a+2\sqrt{bc}\cos\left(\pi\frac{i}{n+1}\right),\,\,\,\,\,\,\text{for}\,\,i=1,\ldots,n,
\end{equation*}
which holds for real and complex $a$, $b$ and $c$ \citeS{smith1985numerical}. Substituting this formula into the expression
\begin{equation*}
    \det\left(\mathbf{G}_{+}\right)=\prod_{i=1}^{H}\prod_{j=1}^{W}\left(\lambda_{i}\left(\mathbf{T}_{2}\right)+\lambda_{j}\left(\mathbf{T}_{1}\right)\right)
\end{equation*}
gives the result in Proposition~\ref{proposition-+-filter-determinant}.
\section*{$+$-Filter Reparameterization}
The following reparameterization is used to ensure that $\mathbf{G}_{+}$ has real positive eigenvalues
\begin{align*}
  a_{1}&=\text{softplus}\left(\rho_{1}\right)+\text{softplus}\left(\rho_{2}\right)\\a_{2}a_{4}&=\left(\text{softplus}\left(\rho_{1}\right)\tanh\left(\rho_{3}\right)/2\right)^{2},\,\,\,\,\,a_{4}/a_{2}=\exp\left(\rho_{4}\right)\\a_{3}a_{5}&=\left(\text{softplus}\left(\rho_{2}\right)\tanh\left(\rho_{5}\right)/2\right)^{2},\,\,\,\,\,a_{5}/a_{3}=\exp\left(\rho_{6}\right),
\end{align*}
where $\rho_1,\ldots,\rho_6$ are real numbers.
\section*{Implementation Details}
 The model parameters $\boldsymbol{\theta}$ and variational parameters $\boldsymbol{\phi}$ are trained with respect to the negative ELBO (Eq.~(\ref{eq:ELBO})) divided by $N$ as loss function, using Adam optimization \citeS{kingma2014adam} with default settings, learning rate $0.01$ and $100$k iterations. The parameters with the lowest loss value are then saved and conditioned on by our implementation of the CG algorithm, for computing the posterior mean and standard deviation of $\mathbf{x}$. We use $N_q=10$ samples from variational approximation to compute the expectation in each iteration. We can train the measurement error $\sigma$ together with the other parameters $\boldsymbol{\theta}$, but we have used a fixed $\sigma=0.001$, which seems to give very similar results, but with somewhat faster convergence. For the DGMRFs with seq-filters, we randomly select among the eight possible orientations of the filters in each layer. As the toy data is centered around 0, the bias in each layer was fixed to 0 for this experiment. The satellite data was normalized to have maximum pixel value 1.
 
\section*{Competing Methods}
We here briefly describe the methods that are compared against in Table~\ref{satellite-table}, except DIP that is mentioned in Section~\ref{sec:DGMs}. For more details we refer to \citet{Heaton2018}.\newline
\textbf{FRK} (Fixed rank kriging) \citeS{zammit2017frk} approximates a spatial process using a linear combination of $K$ spatial basis functions with $K\ll N$.\newline
\textbf{Gapfill} \citeS{gerber2018predicting} is an algorithmic, distribution free method that makes predictions using sorting and quantile regression based on closeby pixels.\newline
\textbf{LatticeKrig} \citeS{nychka2015multiresolution} approximates a GP with a linear combination of multi-resolution basis function with weights that follow a certain GMRF.\newline
\textbf{LAGP} (Local approximate Gaussian process) \citeS{gramacy2015local} fits a GP, but only uses the subset of points in the training data that are closest to the points in the test data.\newline
\textbf{MetaKriging} \citeS{guhaniyogi2017divide} is an approximate Bayesian method that splits the training data into subsets, fits one model to each subset, and combines them all into a meta-posterior, here using GPs.\newline
\textbf{MRA} (Multi-resolution approximation) \citeS{katzfuss2017multi} uses a multi-resolution approximation of a GP, similar to LatticeKrig, but uses compactly supported basis functions.\newline
\textbf{NNGP} (Nearest-neighbor Gaussian process) \citeS{datta2016hierarchical} approximates a GP by rewriting the joint density of the data points as a product of conditional densities, and truncating the conditioning sets to only contain the nearest neighbors.\newline
\textbf{Partition} makes a spatial partitioning (splits the domain into disjoint subsets) and fits spatial basis functions to each partition, similar to FRK, but with some parameters shared between partitions.\newline
\textbf{Pred. Proc.} (Predictive processes) \citeS{finley2009improving} approximates a GP using a set of $K$ knot locations, also known as inducing points, with $K\ll N$ which reduces the size of the covariance matrix that needs to be inverted.\newline
\textbf{SPDE} (Stochastic partial differential equation) \citep{Lindgren2011} represents a GP with a GMRF (see Section~\ref{sec:SPDE}).\newline
\textbf{Tapering} \citeS{furrer2006covariance} obtains an approximation of a GP with sparse covariance matrix by truncating small covariances in a way that preserves positive definiteness.\newline
\textbf{Peri. Embe.} (Periodic embedding) \cite{guinness2017circulant} approximates a GP using the fast Fourier transform on a regular grid.
 
\section*{Linear Trend Model}
For inference with the linear trend model in Eq.~(\ref{eq:linear-trend}), we extend the vector of latents $\mathbf{x}$ to include also the regression coefficients $\boldsymbol{\beta}$, and use (for linear DGMRFs) the prior
\begin{equation*}
\left[\begin{array}{c}
\mathbf{z}\\
\mathbf{z}^{\prime}
\end{array}\right]=\left[\begin{array}{cc}
\mathbf{G} & \mathbf{0}\\
\mathbf{0} & v\mathbf{I}
\end{array}\right]\left[\begin{array}{c}
\mathbf{x}\\
\boldsymbol{\beta}
\end{array}\right]\Leftrightarrow\bar{\mathbf{z}}=\bar{\mathbf{G}}\bar{\mathbf{x}},~~~~~\bar{\mathbf{z}}\sim\mathcal{N}(\mathbf{0},\mathbf{I}),
\end{equation*}
where $v$ can be interpreted as the prior inverse standard deviation of the elements of $\boldsymbol{\beta}$, which we fix at $v=0.0001$. The posterior for $\bar{\mathbf{x}}$ is a GMRF, similar to Eq.~(\ref{eq:posteriorGMRF}), with
\begin{align*}
    \tilde{\mathbf{Q}}&=\bar{\mathbf{G}}^{\top}\bar{\mathbf{G}}+\frac{1}{\sigma^{2}}\left[\begin{array}{c}
\mathbf{I}\\
\mathbf{F}^{\top}
\end{array}\right]\mathbf{I}_{\mathbf{m}}\left[\begin{array}{cc}
\mathbf{I} & \mathbf{F}\end{array}\right],\\ \tilde{\boldsymbol{\mu}}&=\tilde{\mathbf{Q}}^{-1}\left(-\bar{\mathbf{G}}^{\top}\left[\begin{array}{c}
\mathbf{b}\\
\mathbf{0}
\end{array}\right]+\frac{1}{\sigma^{2}}\left[\begin{array}{c}
\mathbf{I}\\
\mathbf{F}^{\top}
\end{array}\right]\mathbf{y}\right),
\end{align*}
and thus we can proceed with inference as before, with $\bar{\mathbf{x}}$ instead of $\mathbf{x}$, with slight modifications to the ELBO and to the CG method. We use an independent variational approximation $q_{\boldsymbol{\phi}_{\boldsymbol{\beta}}}(\boldsymbol{\beta})=\mathcal{N}(\boldsymbol{\beta}|\boldsymbol{\nu}_{\boldsymbol{\beta}},\mathbf{S}_{\boldsymbol{\beta}})$ for $\boldsymbol{\beta}$. Integrating out $\boldsymbol{\beta}$ is important for the predictive performance. For reference, if linear trends are instead removed using the ordinary least squares estimates of $\boldsymbol{\beta}$ in a preprocessing step, the row in Table~\ref{satellite-table} corresponding to $\text{seq}_{5\times5,L=5}$ instead reads (1.25, 1.74, 0.90, 8.45, 0.89). When the linear trend model is used, we compute posterior standard deviations using standard Monte Carlo estimates, instead of simple RBMC, using $N_s=100$ samples.
\begin{figure*}[ht]
\begin{center}
\includegraphics[width=80mm,trim={0mm 2mm 0 1mm},clip]{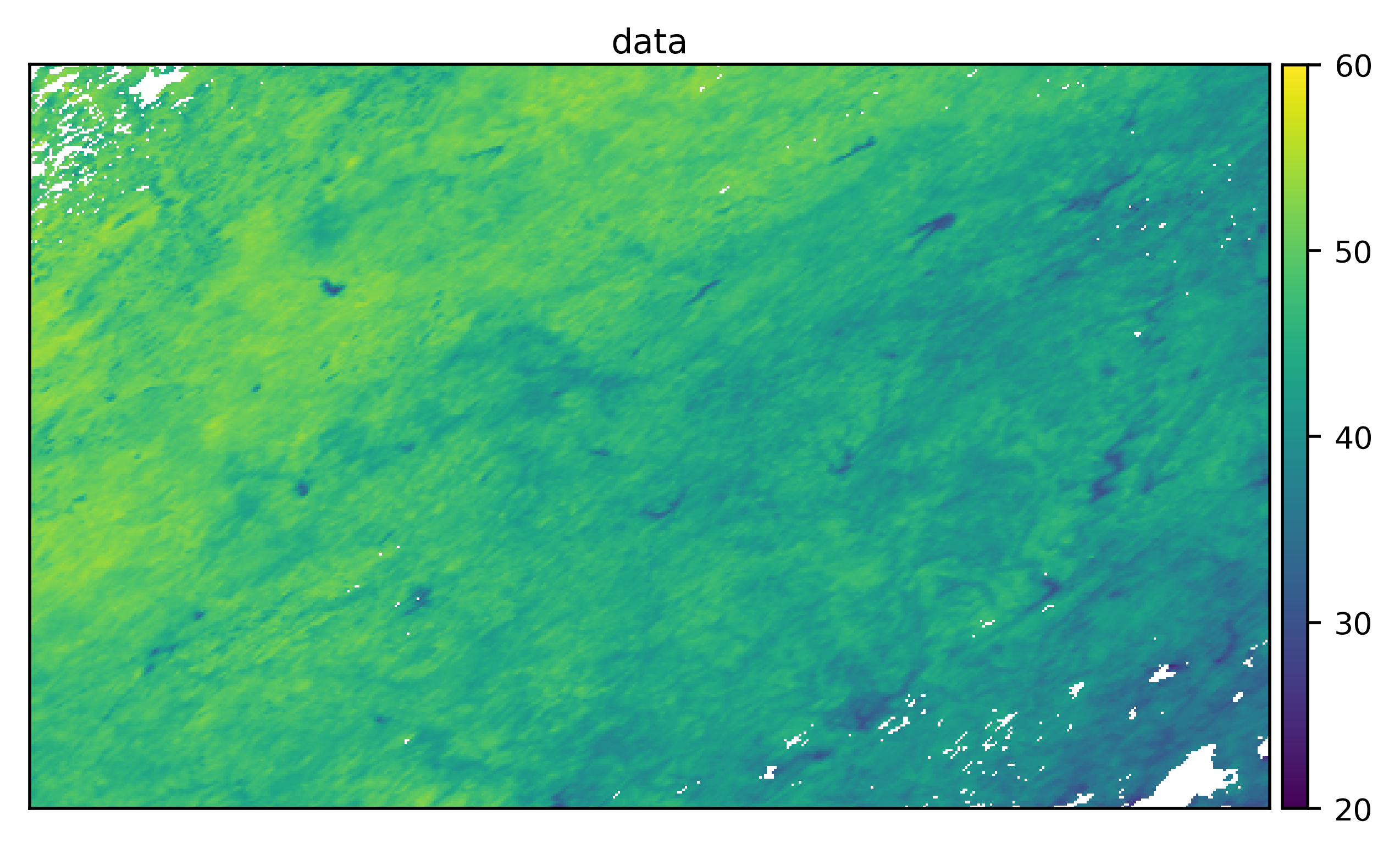}
\includegraphics[width=80mm,trim={0mm 2mm 0 1mm},clip]{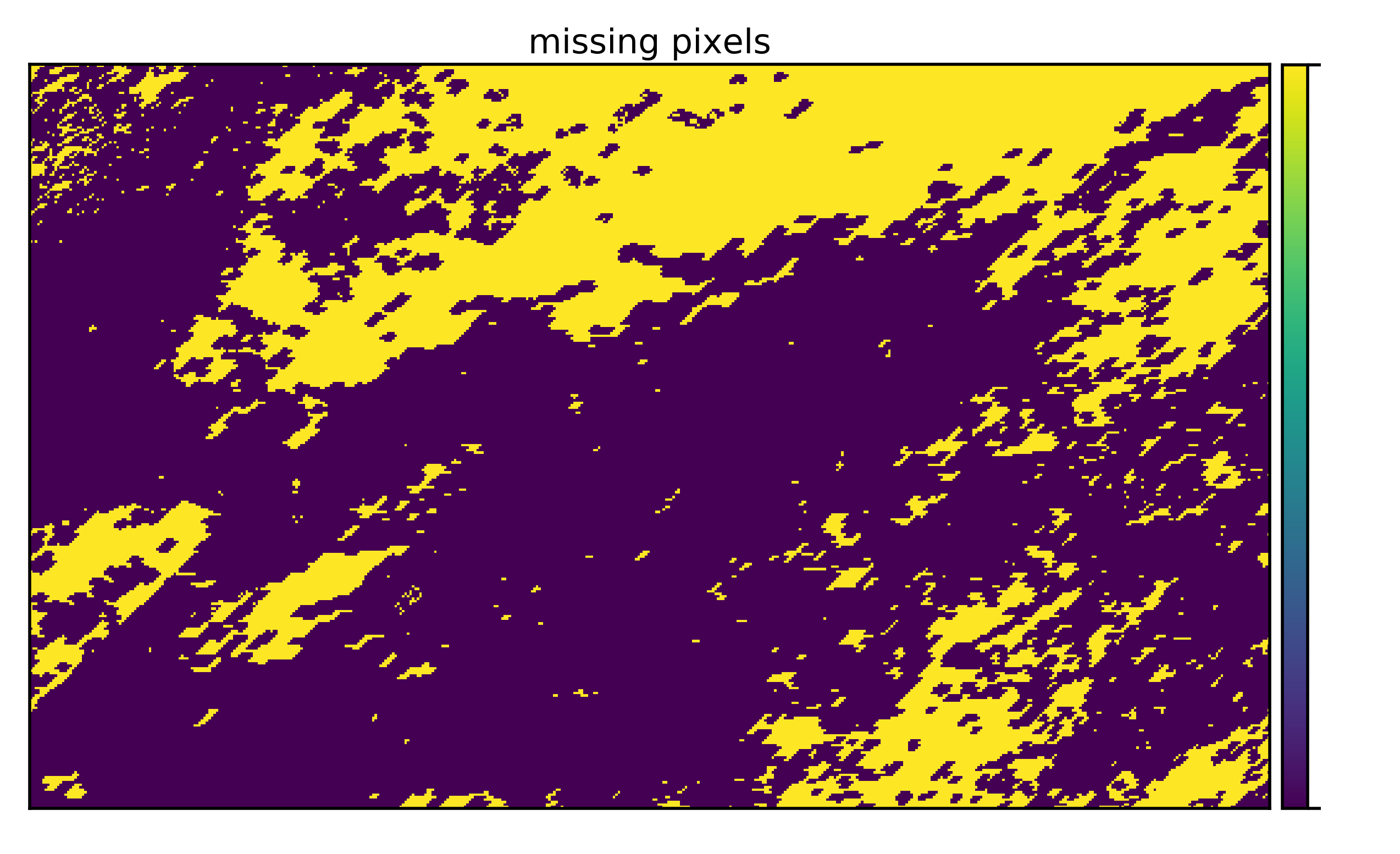}
\includegraphics[width=80mm,trim={0 2mm 0 1mm},clip]{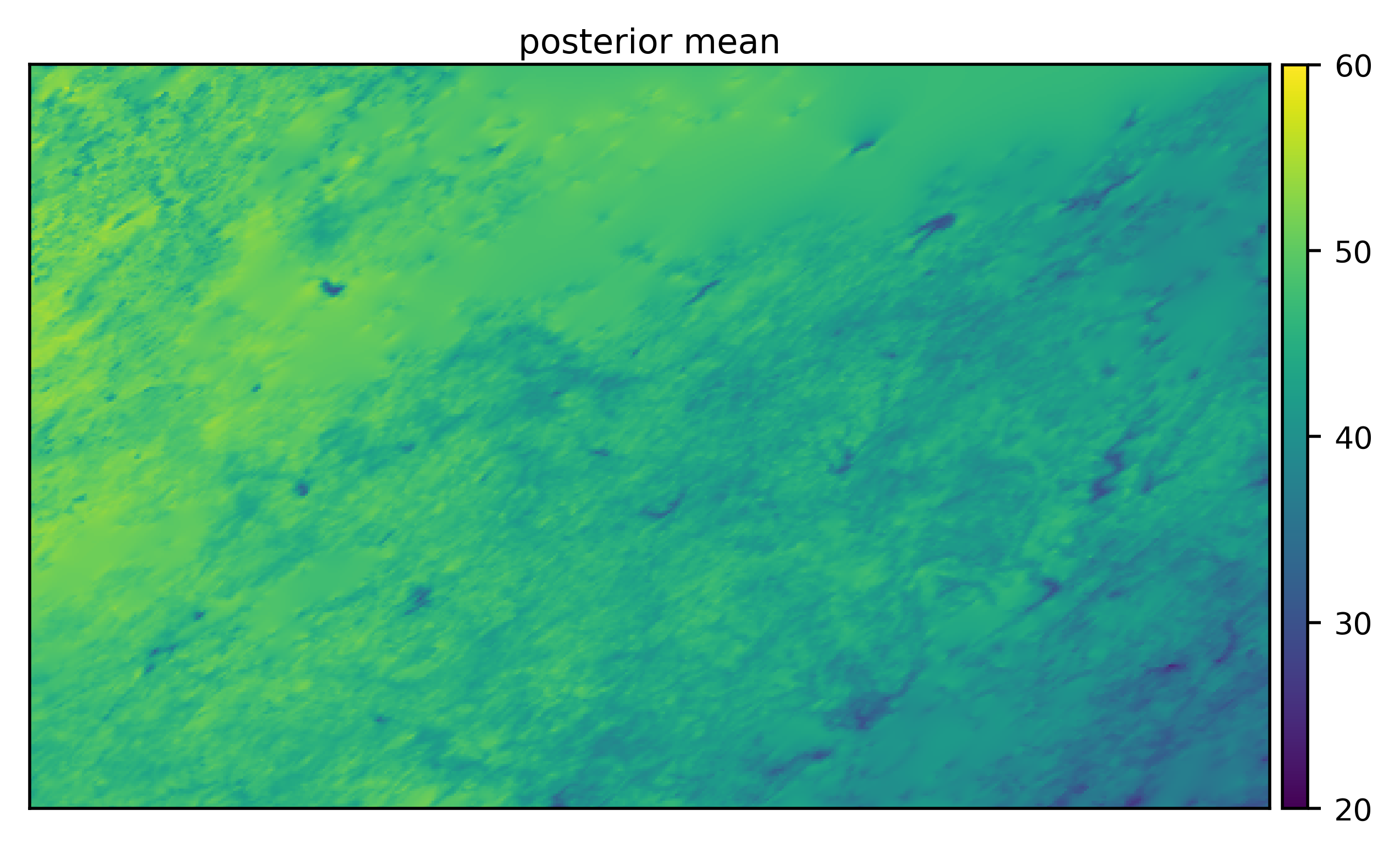}
\includegraphics[width=80mm,trim={0 2mm 0 1mm},clip]{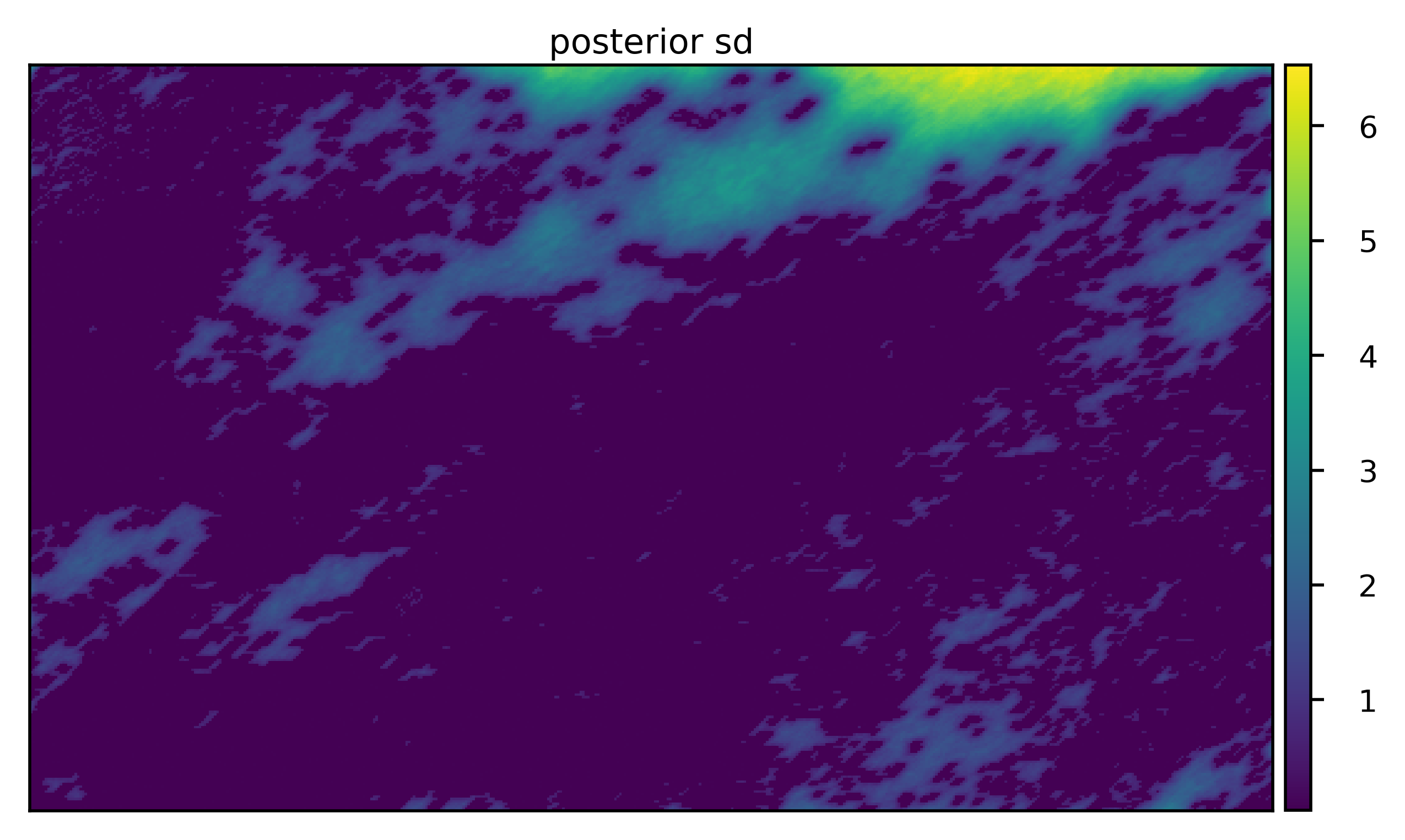}
\includegraphics[width=80mm,trim={0 2mm 0 1mm},clip]{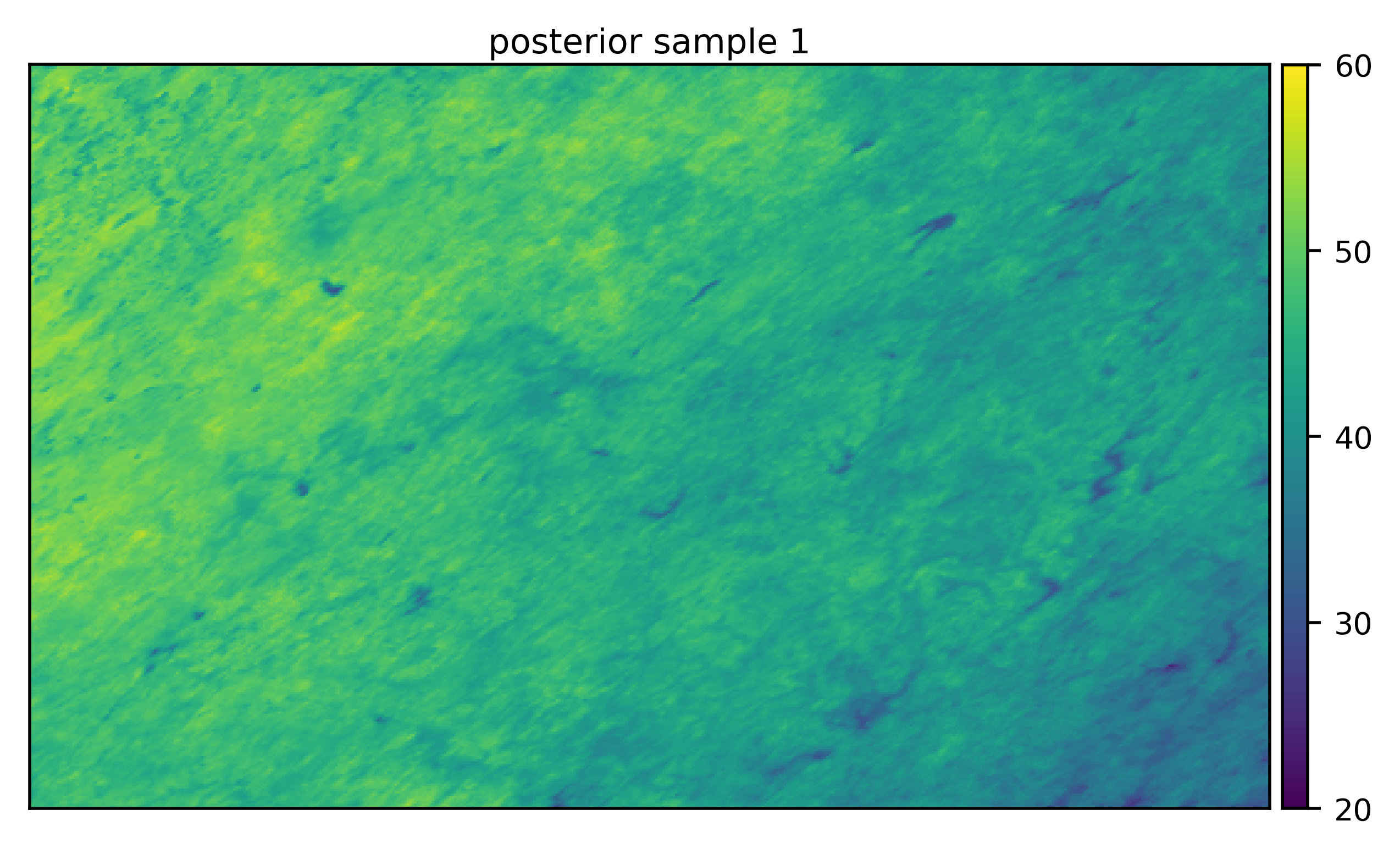}
\includegraphics[width=80mm,trim={0 2mm 0 1mm},clip]{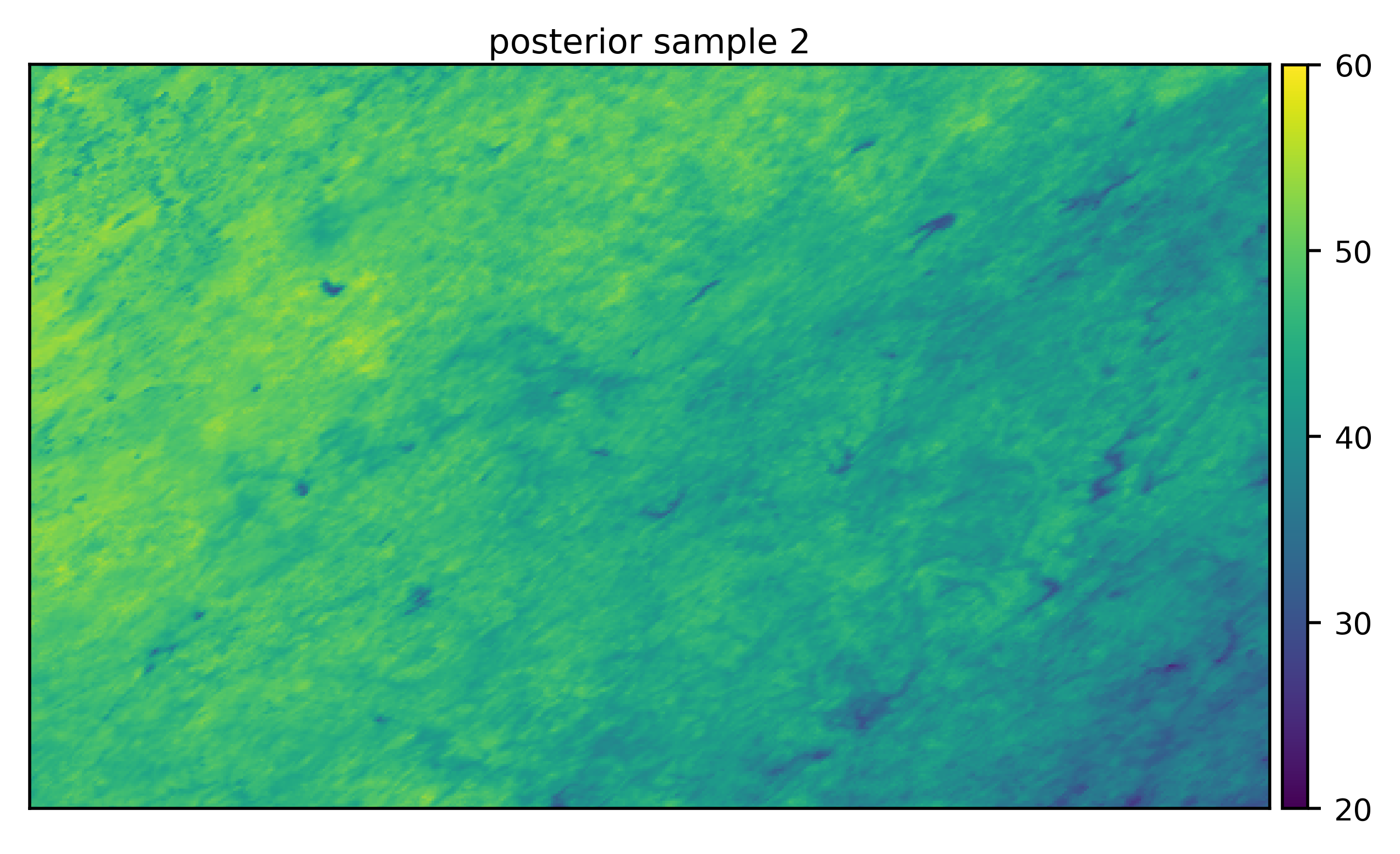}
\caption{Satellite data inpainting by a linear DGMRF with 5 layers of 5$\times$5 seq-filters.}\label{fig:satellite-data}
\end{center}
\end{figure*}
\begin{table*}[ht]
\caption{Standard deviations across seeds for the results in Table~\ref{satellite-table}.}
\label{satellite-table-sd}
\centering
\small
\begin{tabular}{rlllll}
\toprule
Method & MAE & RMSE & CRPS & INT & CVG \\
\midrule
  $\text{seq}_{5\times5,L=1}$ & 0.029 & 0.040 & 0.011 & 0.216 & 0.000 \\ 
  $\text{seq}_{5\times5,L=3}$ & 0.022 & 0.042 & 0.019 & 0.462 & 0.001 \\ 
  $\text{seq}_{5\times5,L=5}$ & 0.037 & 0.051 & 0.012 & 0.461 & 0.001 \\
  $\text{seq}_{3\times3,L=5}$ & 0.066 & 0.097 & 0.039 & 0.171 & 0.003 \\ 
  $+_{L=5}$ & 0.039 & 0.056 & 0.018 & 0.221 & 0.001 \\ 
  $\text{seq}_{5\times5,L=5,\text{NL}}$ & 0.066 & 0.092 & - & - & - \\
\bottomrule
\end{tabular}
\end{table*}
\clearpage
\bibliographystyleS{icml2020}
\bibliographyS{paper}
\end{document}